\documentclass[12pt]{article}

\usepackage[utf8]{inputenc}
\usepackage[T1]{fontenc}
\usepackage{amsmath, amssymb} 
\usepackage{graphicx} 
\usepackage{url} 
\usepackage{natbib} 

\usepackage{bbm}
\usepackage{bm}
\usepackage{multirow}

\usepackage[font=small, labelfont=bf]{caption}

\usepackage{amsthm}
\newtheorem{theorem}{Theorem}[section]

\newtheorem{corollary}{Corollary}[theorem]

\usepackage{setspace}
\providecommand{\keywords}[1]
{
  \small	
  \textbf{\textit{Keywords---}} #1
} 

\newcommand{\authorname}{Francisco~J.~P{\'e}rez-Reche}
\newcommand{\authoremail}{fperez-reche@abdn.ac.uk}
\newcommand{\authoraddress}{School of Natural and Computing Sciences, University of Aberdeen, Fraser Noble Building, AB24 3UE, United Kingdom}

\makeatletter
\renewcommand{\@maketitle}{
    \begin{center}
        \vskip 1em
        {\Large \bfseries \@title \par}
        \vskip 1.5em
        {\Large \authorname \par}
        \vskip 0.5em
        {\em \authoraddress \par}
        \vskip 0.5em
        {\ttfamily \authoremail \par}
        \vskip 2em
    \end{center}
}
\makeatother

\title{The elbow statistic: Multiscale clustering statistical significance}

\begin{document}

\maketitle

\begin{abstract}
Selecting the number of clusters remains a fundamental challenge in unsupervised learning. Existing approaches typically focus on identifying a single ``optimal'' partition, often overlooking statistically meaningful structure present across multiple resolutions. 
We introduce ElbowSig, a general inferential framework for assessing clustering structure over a range of resolutions. The method formalizes the elbow heuristic by defining a normalized discrete curvature statistic based on the sequence of within-cluster heterogeneity values, and evaluates its significance relative to a null distribution of unstructured data. This yields hypothesis tests across resolutions, enabling simultaneous inference at multiple clustering scales.
We derive the asymptotic behavior of the null statistic in both large-sample and high-dimensional regimes, characterizing its limiting form and variability. Because it depends only on the heterogeneity sequence, ElbowSig is compatible with a wide range of clustering algorithms, including hard, fuzzy, and model-based methods.
Experiments on synthetic and real datasets show that the procedure controls Type-I error under unstructured data while providing power to detect multiscale organization, revealing structure that is often missed by single-resolution selection criteria.
\end{abstract}

\keywords{Cluster number selection, Elbow statistic, Multiscale inference, Non-parametric testing, Unsupervised learning}

\section{Introduction}

Clustering partitions observations into groups of similar points and is central to statistics and unsupervised learning, with applications across many domains. Common approaches include $k$-means, hierarchical methods, and density-based algorithms \citep{aggarwal_data_2014,Gan-Ma-Wu_DataClustering_Book2007, Rodriguez2019}, all of which require selecting an appropriate number of clusters to characterize the underlying structure.

For a fixed clustering algorithm, many criteria have been proposed to select the number of clusters \citep{Milligan1985,Gordon1999Classification}. Widely used examples include the Davies--Bouldin \citep{Davies1979}, Calinski--Harabasz \citep{Calinski-Harabasz}, and silhouette indices \citep{Rousseeuw1987}, which target a single optimal value $\hat{k}$. The elbow method offers an alternative approach: instead of maximizing a score, it identifies the number of clusters $\hat{k}$ where the within-cluster heterogeneity $H_k$ shows the strongest slope change (i.e. seeks an elbow in the $H_k$ versus $k$ curve) \citep{Sugar2003,Willmott_MachineLearning2009}. Despite its popularity, elbow detection is often based on visual inspection and therefore lacks a formal inferential interpretation. 

More broadly, many procedures to identify the number of clusters in a dataset tend to return $\hat{k}>1$ even when the data are effectively unstructured. The gap statistic addresses this by comparing observed heterogeneity reduction with that from null reference datasets \citep{tibshirani_estimating_2001}, allowing selection of $k\ge 1$, including $k=1$.

Other methods provide formal hypothesis tests for clustering structure, but they typically focus on specific decision problems, such as testing a single split or selecting a single optimal number of clusters. For example, \citet{Maitra2012} propose a bootstrap-based test to assess whether a partition into $k^*$ clusters is stronger than expected under a null model consisting of $k<k^*$ spherical or ellipsoidal clusters. SigClust \citep{Liu2008} provides another example, testing whether a two-cluster split is stronger than expected under a single Gaussian null; extensions allow for broader classes of unimodal null distributions \citep{Shen2024}. Although these approaches provide rigorous inferential tools, they are not designed to quantify statistical evidence across a sequence of resolutions, nor to compare the relative significance of multiple values of $k$ within a unified framework.

While these approaches introduce a statistical notion of significance into clustering, they typically focus on evaluating a single partition or testing a specific split. As a result, they do not provide a mechanism for assessing how statistical evidence for clustering structure evolves across different resolutions. This is a significant limitation, as many datasets exhibit meaningful organization at multiple levels of granularity, rather than admitting a single, well-defined number of clusters.

Methods such as pvClust \citep{Suzuki2006} and SHC \citep{Kimes2017} partially address this issue by assessing significance within hierarchical clustering. However, these approaches remain tied to specific clustering paradigms and do not offer a general framework for evaluating the statistical significance of arbitrary partitions across resolutions.


We introduce ElbowSig, a general inferential framework for assessing statistical evidence of clustering structure across multiple resolutions. Rather than selecting a single optimal number of clusters, ElbowSig quantifies the significance of curvature changes in the heterogeneity sequence $H_k$, enabling simultaneous inference across all values of $k$. 
Because ElbowSig depends only on the heterogeneity sequence, it is compatible with a wide range of clustering algorithms, enabling a unified inferential treatment of multiscale clustering structure independent of the underlying partitioning method.

The rest of the article is organized as follows. Section~\ref{sec:ElbowStatistic} introduces the elbow statistic and its interpretation as a multiscale signature of structural transitions. Section~\ref{sec:Baseline_Unstructured} derives baseline reference behavior for unstructured data in large-sample and large-dimensional regimes. We then present the ElbowSig computational framework and associated hypothesis-testing procedure in Section~\ref{sec:HypothesisTesting} , followed by applications to synthetic clustered data (Section~\ref{sec:SyntheticClusteredData}), random unstructured synthetic data (Section~\ref{sec:UnstructuredData}), and real datasets (Section~\ref{sec:RealData}).

\section{The elbow statistic}
\label{sec:ElbowStatistic}
Let $\mathcal{X} = \{x_i\}_{i=1}^N$ denote a dataset of $N$ points in $D$ dimensions.  
For a given number of clusters $k$, a clustering algorithm produces a partition 
$\{C_j\}_{j=1}^k$ (hard or fuzzy) together with $M$ representative parameters 
$\bm{\theta} = \{\theta_j\}_{j=1}^M$.  To each partition, we associate a scalar measure of 
within-cluster heterogeneity $H_k$ quantifying the adequacy of the partition. The heterogeneity $H_k$ is a non-negative and non-increasing function of the number of clusters which can be defined as
\begin{equation}
\label{eq:Hk_definition}
  H_k = \sum_{i=1}^N h_k(x_i, \bm{\theta})~,
\end{equation}
where $h_k(x_i, \bm{\theta})$ is a non-negative function that quantifies the contribution of point $x_i$ to the overall heterogeneity of the partition. 
Table~\ref{Table1} presents representative choices of $h_k$ that are appropriate for the clustering methods considered in this work: agglomerative clustering, k-means, fuzzy c-means (FCM), and Gaussian mixture models (GMMs). More generally, ElbowSig is not tied to these specific forms; it applies to any heterogeneity function $h_k(x_i,\bm{\theta})$ that induces a well-defined, non-increasing heterogeneity sequence $H_k$.

Empirically, datasets with clustered structure display a marked change in the slope of $H_k$ as $k$ increases: a region of rapid decay for $k < k^*$ followed by a much slower decay beyond $k^*$.
The inflection (or ``elbow'') at $k^*$ signals the transition between two regimes:
\begin{itemize}
  \item[(i)] for $k \le k^*$, adding new clusters substantially reduces
  heterogeneity by separating genuinely distinct groups in the data;
  \item[(ii)] for $k > k^*$, additional clusters mainly subdivide already cohesive
  regions, yielding diminishing returns in $H_k$.
\end{itemize}

\begin{table}[t]
\centering
\caption{
Point-wise heterogeneity measures $h_k(x_i;\bm{\theta})$ for the clustering models considered in this paper. For hard clustering (agglomerative and k-means) and FCM,  $\mu_j$ denotes the centroid of cluster $C_j$. For FCM, $w_{ij} \in [0,1]$ are the membership weights satisfying $\sum_{j=1}^k w_{ij} = 1$, $m>1$ is the fuzzifier exponent. 
For GMM, each component has mixing weight $\pi_j$ ($\sum_{j=1}^k \pi_j = 1$), mean $\mu_j$, and covariance $\Sigma_j$; $\phi(x_i \mid \mu_j, \Sigma_j)$ denotes the corresponding Gaussian density. 
\label{Table1}}
\vspace{0.7em}

\small
\begin{tabular}{|l|l|l|}
\hline
\textbf{Clustering method}
& \textbf{Point-wise heterogeneity $h_k(x_i;\bm{\theta})$}
& \textbf{Parameters $\bm{\theta}$}
\\ \hline

\parbox{4.5cm}{
Hard clustering \\[4pt]
(agglomerative and k-means)
}
&
\parbox{6cm}{
Inertia: \\[4pt]
$
\min_{1 \le j \le k}
\|x_i - \mu_j\|^2
$
}
&
\parbox{3cm}{
$\displaystyle
\mu_j
$
}
\\ \hline

Fuzzy $c$--means  (FCM)
&
\parbox{6cm}{
Weighted inertia: \\[4pt]
$\displaystyle
\sum_{j=1}^k
w_{ij}^m\,
\|x_i - \mu_j\|^2
$
}
&
\parbox{3cm}{
$
\displaystyle
(\mu_j,m,\{w_{ij}\})
$
}
\\ \hline
Gaussian mixture model (GMM)
&
\parbox{6cm}{
Negative log-likelihood: \\[4pt]
$\displaystyle
- \ln\!\Big(
\sum_{j=1}^k
\pi_j\,\phi(x_i\!\mid\!\mu_j,\Sigma_j)
\Big)
$
}
&
\parbox{3cm}{
$\displaystyle
(\pi_j,\mu_j,\Sigma_j)
$
}
\\ \hline

\end{tabular}
\end{table}

For data with hierarchical structure where groups themselves may contain finer subgroups, $H_k$ may exhibit several elbows. Indeed, as $k$ increases, the optimal partition first allocates clusters to capture coarse components, producing a first elbow at $k_1^*$.  
Further increases in $k$ resolve substructure within these components, producing secondary elbows $(k_2^*, k_3^*, \dots)$.  
Each elbow thus corresponds to a newly resolved organizational scale. This is analogous to the phase transitions in clustering identified by \citet{Rose1990} and \citet{Rose1998}.

To quantify the local change in slope, we define the \emph{elbow statistic},
\begin{equation}
\label{eq:deltak_definition}
  \delta_k = -\frac{\Delta^2 H_k }{\Delta H_{k}}~,
\end{equation}
where 
$\Delta H_k = H_{k+1} - H_k$ and $\Delta^2 H_k = \Delta H_k - \Delta H_{k-1}$
are the first and second discrete differences of $H_k$, respectively.

The elbow statistic $\delta_k$ is a normalized discrete analogue of the
\emph{second derivative} of $H_k$.
An elbow in $H_k$ corresponds to a point of maximum curvature, manifesting as a local maximum of $\delta_k$.
Peaks in $\delta_k$ therefore indicate scales at which the rate of heterogeneity reduction undergoes an abrupt change,
signalling the emergence of new structure in the data.

The sequence $\{\delta_k\}$ can be interpreted as a quantitative signature of structural transitions in the data.  
Plotting $\delta_k$ as a function of $k$ offers a direct analogue to the classical
``elbow plot'' of $H_k$ versus $k$, but with enhanced interpretability:  
peaks in $\delta_k$ correspond to local maxima in the curvature of $H_k$
and thus mark scales where the rate of heterogeneity reduction changes most abruptly.
In practice, these peaks reveal the emergence of new substructures or cluster
levels within the data hierarchy.

However, not every local maximum in $\delta_k$ necessarily reflects meaningful
organization: random spatial irregularities or finite-sample fluctuations
can produce spurious peaks even in unstructured data.  
To distinguish genuine structural transitions from such random effects,
it is essential to compare the observed $\delta_k$ sequence against
a reference baseline $\delta_k^{(r)}$ derived from unstructured data
(e.g.\ uniformly distributed samples with the same support and covariance as the original data).
The next sections motivate this comparison by deriving an asymptotic
baseline for random data and introduce a practical numerical
procedure for evaluating deviations from this baseline.

\section{Baseline reference $\delta_k^{(r)}$ for unstructured data}
\label{sec:Baseline_Unstructured}

In this section, we study the dependence of the elbow statistic $\delta_k^{(r)}$ for unstructured data on the sample size $N$, the number of clusters, $k$, and data dimensionality, $D$. 
 More specifically, asymptotic results are presented in two complementary limits: large-sample ($N \rightarrow \infty$ with fixed $D$ and $k$) and large dimensionality ($D \rightarrow \infty$ with fixed $N$ and $k$).
The resulting characterizations of $\delta_k^{(r)}$ provide the theoretical foundation for the ElbowSig framework introduced in Section~\ref{sec:HypothesisTesting}, and underpin its application to clustered data in Sections~\ref{sec:SyntheticClusteredData}--\ref{sec:RealData}.

\subsection{Large-sample asymptotics}
\label{sec:AsymptoticBaseline_N}

\begin{theorem}
\label{thm:AsymptoticBaseline_N}
Let the heterogeneity $H_k$ be defined as in Eq.~\eqref{eq:Hk_definition}, and suppose that the data are drawn i.i.d. from a distribution $f$ on $\mathbb{R}^D$ with finite second moment. For fixed parameters $\bm{\theta}$ and fixed $k$, define the elbow statistic $\delta_k^{(r)}$ based on a reference sample of size $N$.

Then, as $N \to \infty$,
\begin{equation}
\label{eq:Edeltak_Nlimit}
\mathbb{E}[\delta_k^{(r)}]
=
-\frac{\Delta^2 A_k}{\Delta A_k}
+
O(N^{-1}),
\qquad
\mathrm{Var}(\delta_k^{(r)}) = O(N^{-1}),
\end{equation}
provided $\Delta A_k \neq 0$, where
\begin{equation}
\label{eq:Ak}
A_k = \int_{\Omega} h_k(x,\bm{\theta}) f(x)\,dx
\end{equation}
denotes the expected point-wise heterogeneity.
\end{theorem}

\begin{proof}
Since, for fixed $\bm{\theta}$, the reference data heterogeneity $H_k^{(r)}$ can be written as a sum of i.i.d. point-wise heterogeneity functions (Eq.~\eqref{eq:Hk_definition}), the central limit theorem implies 
\citep{Grimmett2001}
\begin{equation}
H_k = N A_k + O_p(N^{1/2}),
\end{equation}
where $O_p$ denotes stochastic order and $A_k$ is given by Eq.~\eqref{eq:Ak}.

Consequently,
\begin{equation}
\label{eq:E_Var_Hk_LargeN}
\mathbb{E}[H_k] = N A_k,
\qquad
\mathrm{Var}(H_k) = O(N),    
\end{equation}
so that the empirical heterogeneity converges to its population counterpart 
at rate $N^{-1/2}$.

Applying the delta method \citep{VerHoef2012} to the smooth mapping 
$(x,y)\mapsto -x/y$ defining the elbow statistic (Eq.~\ref{eq:deltak_definition}) yields Eq.~\eqref{eq:Edeltak_Nlimit} and finalises the proof.
\end{proof}

Theorem~\ref{thm:AsymptoticBaseline_N} shows that the elbow statistic concentrates around a deterministic population functional determined by the sequence
$\{A_k\}$, with fluctuations of order $N^{-1/2}$. 

\begin{corollary}
For unstructured reference data with $f(x)=|\Omega|^{-1}$, the expected value of the elbow statistic satifies 
\begin{equation}
  \label{eq:delta_k_asymptotic_N}
\mathbb{E}[\delta_k^{(r)}]
\sim
\frac{1 + 2/D}{k},
\end{equation}
in the large-sample regime $N \to \infty$ followed by $k \rightarrow \infty$, with fixed dimension $D$.
\end{corollary}

\begin{proof}
For $f(x)=|\Omega|^{-1}$, the clustering problem can be described as an optimal vector quantization problem. Under standard regularity conditions, high-rate quantization theory \citep{Graf2000} implies
\begin{equation}
A_k \simeq C_0 + C_1 k^{-2/D},
\end{equation}
for large $k$, where $C_0$ and $C_1$ are independent of $k$. Substituting this asymptotic form into Eq.~\eqref{eq:Edeltak_Nlimit} yields Eq.~\eqref{eq:delta_k_asymptotic_N}.
\end{proof}

Hence, under the null model corresponding to unstructured data, the expected baseline elbow decreases 
smoothly as $k^{-1}$, with a prefactor depending only on the 
dimension $D$. Systematic deviations of the observed $\delta_k$ from this smooth
reference behavior indicate the presence of genuine cluster structure.

\subsection{Asymptotic behavior for high dimensionality}
\label{sec:AsymptoticBaseline_D}

We now consider the complementary asymptotic regime in which the
dimension $D \to \infty$ for fixed $N$ and $k$.

\begin{theorem}
\label{thm:AsymptoticBaseline_D}
Let $x_1,\dots,x_N \in \mathbb{R}^D$ be i.i.d. from an isotropic sub-Gaussian distribution~\citep{Vershynin2018} with
\[
\mathbb{E}[x_i]=0,
\qquad
\mathbb{E}[x_i x_i^\top]=\sigma^2 I_D.
\]

Let the heterogeneity $H_k$ be defined as in Eq.~\eqref{eq:Hk_definition}, and assume that, for fixed parameters $\bm{\theta}$, the pointwise heterogeneity $h_k(x_i,\bm{\theta})$ is a measurable function of $x_i$ that depends on $x_i$ through quadratic forms (or smooth functions thereof), so that it inherits sub-Gaussian concentration at scale $D^{1/2}$.

Then, for fixed $k$ and $\bm{\theta}$, the elbow statistic satisfies, as $D \to \infty$,
\begin{equation}
\mathbb{E}[\delta_k^{(r)}]
=
-\frac{\Delta^2 B_k}{\Delta B_k}
+
O(D^{-1/2}),
\qquad
\mathrm{Var}(\delta_k^{(r)}) = O(D^{-1}),
\end{equation}
provided $\Delta B_k \neq 0$, where $B_k$ is defined through
\begin{equation}
\label{sec:Hk_highD}
H_k = D B_k + O_p(D^{1/2}).
\end{equation}
\end{theorem}

\begin{proof}
Under the stated isotropic sub-Gaussian assumption, standard results on concentration of quadratic forms~\citep{Vershynin2018} imply that
\[
\|x_i\|^2 = \sigma^2 D + O_p(\sqrt{D}),
\qquad D \to \infty.
\]

Since the centroid parameters $\mu_j$ are finite linear combinations of the vectors $x_i$ (with $N$ fixed), they inherit sub-Gaussian concentration in each coordinate. From this, it follows that forms of the type $\|x_i-\mu_j\|^2$ also concentrate at rate $O_p(\sqrt D)$.

As a consequence, the heterogeneity admits the expansion given by Eq.~\eqref{sec:Hk_highD} with a deterministic coefficient $B_k$ that depends on the specific definition of heterogeneity.

Applying the delta method to the smooth mapping defining $\delta_k^{(r)}$ (Eq.~\ref{eq:deltak_definition}) yields the stated expansion for the mean and variance.
\end{proof}

Theorem \ref{thm:AsymptoticBaseline_D} applies to clustering methods based on squared Euclidean distances (e.g., $k$-means, FCM, and Gaussian mixture models), for which the heterogeneity function inherits the concentration properties of quadratic forms under the sub-Gaussian assumption. Consequently, in the high-dimensional regime, the baseline elbow statistic associated with this class of heterogeneity functions concentrates around a deterministic limit, with variance of order $D^{-1}$.

The $D^{-1}$ decay of the variance holds for all the clustering methods considered here. Fig.~\ref{fig:Stats_vs_D}(b) demonstrates this numerically. In contrast, the asymptotic value of $\mathbb{E}[\delta_k^{(r)}]$ depends on the functional form of $B_k$ which, in turn, depends on the heterogeneity definition as follows:

\begin{itemize}
\item For hard-clustering inertia, $B_k = D \sigma^2 (N-k)$, so that $\Delta^2 B_k = 0$ and
$\mathbb{E}[\delta_k^{(r)}] = O(D^{-1})$. Thus, the baseline elbow vanishes asymptotically, as demonstrated numerically in Fig.~\ref{fig:Stats_vs_D}(a) for the agglomerative and k-means methods based on the hard-clustering inertia.

\item For FCM, analytical expressions can be derived in the regime of large $m$, where the membership weights $w_{ij}$ approach $k^{-1}$ \citep{Hathaway1988}. In this limit, $B_k = \sigma^2 N k^{1-m}$,  and the expected elbow statistic converges, as $D\rightarrow \infty$, to
\begin{equation}
\label{eq:delta_k_asymptotic_FCM}
\mathbb{E}[\delta_k^{(r)}] \xrightarrow[D\to\infty]{} \frac{k^{1-m}-(k-1)^{1-m}}{(k+1)^{1-m}-k^{1-m}}-1
\end{equation}

The numerical results in Fig.~\ref{fig:Stats_vs_D}(a) indicate that this asymptotic approximation is remarkably accurate even for $m=2$ and $D<10$. In particular, the dotted line corresponding to Eq.~\eqref{eq:delta_k_asymptotic_FCM} closely matches the simulated values (diamonds). Moreover, the very small variance observed in Fig.~\eqref{fig:Stats_vs_D}(b) shows that the baseline elbow statistic for FCM is tightly concentrated around the value predicted by Eq.~\eqref{eq:delta_k_asymptotic_FCM}, even at moderate dimensions.

\item For Gaussian mixture models,
$B_k$ exhibits a logarithmic dependence on $k$ which results in
\begin{equation}
\mathbb{E}[\delta_k^{(r)}] \xrightarrow[D\to\infty]{} -\frac{\ln (1-k^{-2})}{\ln (1+k^{-1})} \sim \frac{1}{k} \qquad (k \to \infty).
\end{equation}
While this predicts an asymptotic trend to a constant positive $k$-dependent value, the observed behavior in  Fig.~\ref{fig:Stats_vs_D}(a) is masked by significant fluctuations. Specifically, Fig.~\ref{fig:Stats_vs_D}(b) reveals a high variance in $\delta_k$ when GMMs are applied to finite simulated datasets. This variability likely stems from the non-convexity of the mixture likelihood surface, which is prone to multiple local maxima \citep{Murphy_BookMachineLearning,Chen2024}. As a result, maximum‑likelihood estimates  can differ substantially across independent datasets and random initializations, particularly in small samples.
\end{itemize}

\begin{figure}[!t]%
\centering
\includegraphics[width=0.8\linewidth]{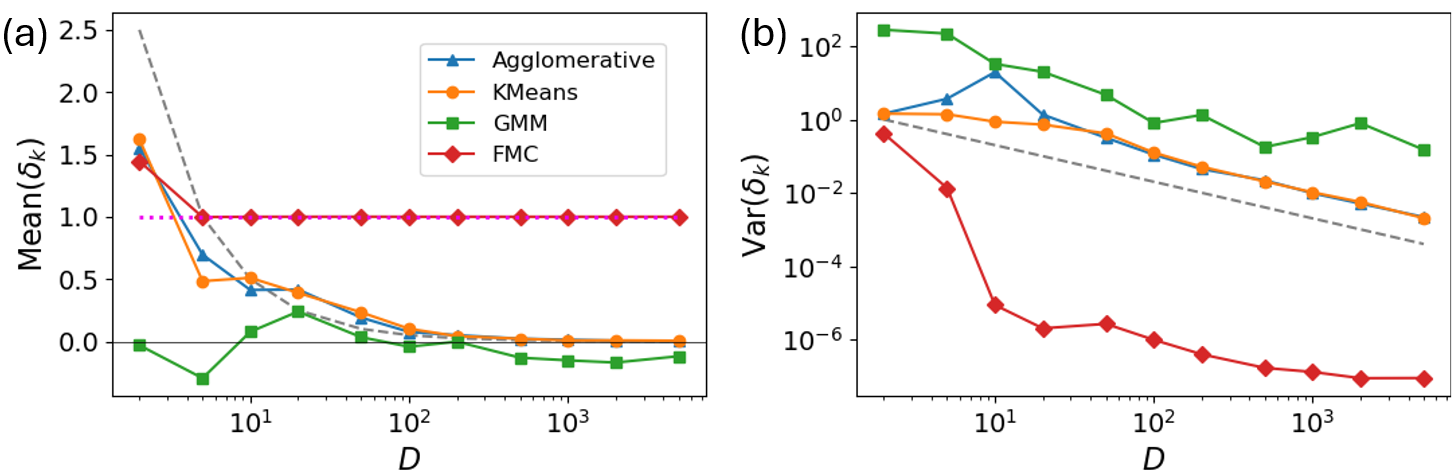}
\caption{Mean value and variance of the elbow statistic $\delta_k^{(r)}$ for unstructured data as a function of the dimension $D$ for agglomerative, k-means, GMM and FCM clustering methods. The expectation and variance are computed from $N_R = 500$ reference datasets of size $N=30$ uniformly distributed on $[0,1]^D$, assuming $k=3$ clusters. The dashed lines indicate the predicted asymptotic trends for large $D$. The dotted line in panel (a) indicates the asymptotic value of $\mathbb{E}[\delta_k^{(r)}]$ for FCM in the limit of large fuzzifier parameter $m$. 
\label{fig:Stats_vs_D}
}
\end{figure}

\section{The ElbowSig framework: Computation of the elbow statistic and clustering hypothesis testing}
\label{sec:HypothesisTesting}

The previous section shows that the baseline elbow statistic $\delta_k^{(r)}$ for unstructured data exhibits non-negligible fluctuations, which are particularly pronounced for small sample sizes $N$ and low $D$. Therefore, to properly assess whether an observed slope change $\delta_k^{\mathrm{data}}$ reflects genuine structure rather than a random fluctuation under the null hypothesis of unstructured data, one should take these fluctuations into account when comparing the observed elbow statistic with the baseline. ElbowSig implements a practical procedure for this comparison, which consists of four steps.

\paragraph{(1) Elbow statistic for the observed data.}
Given a clustering method and heterogeneity measure, we compute the sequence
$\{\delta_k^{\mathrm{data}}\}_{k=1}^{k_{\max}}$ for the observed dataset $\mathcal{X}$.

\paragraph{(2) Elbow statistics under the null hypothesis.}
To evaluate statistical significance, we generate $N_R$
reference datasets, $\{\mathcal{X}^{(r)}\}_{r=1}^{N_R}$, of the same size as $\mathcal{X}$ that contain no cluster structure.
For each reference dataset, we compute the baseline elbow statistic
$\{\delta_k^{(r)}\}_{k=1}^{k_{\max}}$, yielding an empirical null
distribution for every~$k$.

Two choices of reference generator are considered, following \cite{tibshirani_estimating_2001} proposal for the gap statistic:
\begin{itemize}
\item[(i)] \emph{Bounding-box uniformity}: features are sampled independently and uniformly over their observed ranges.
\item[(ii)] \emph{PCA-aligned uniformity}: reference points are generated uniformly in a PCA-aligned hyperrectangle and mapped back to the original coordinate system.
\end{itemize}

\paragraph{(3) Empirical p-values.}
For each scale $k$,we assess the extremity of the observed elbow statistic $\delta_k^{\mathrm{data}}$ relative to the null distribution. An empirical $p$-value is computed as:
\begin{equation}
p_k = 
\frac{1}{N_R}
\sum_{r=1}^{N_R}
\mathbbm{1}\!\left(\delta_k^{(r)} \ge \delta_k^{\mathrm{data}}\right).
\end{equation}

\paragraph{(4) Significance criteria.} 
To identify statistically meaningful structures while controlling for Type I errors (false positives), we employ two complementary assessment levels:
\begin{itemize}
\item[(i)] Individual scale significance criterion (\emph{``per-scale''} or \emph{``per-$k$''}): A partition into $k$ clusters is declared to contain statistically meaningful structure at level $q_1$ if $p_k < p_{\text{sig}}(q_1)$. The threshold $p_{\mathrm{sig}}(q_1)$ is calibrated to ensure
that, under the null hypothesis of unstructured data, the probability of (incorrectly)
declaring \emph{any} scale $k \in \{1,\dots,k_{\text{max}}\}$ as being significant is bounded by~$q_1$. We define this conservative, scale-independent threshold as $p_{\mathrm{sig}}(q_1) = \min_{1 \le k \le k_{\max}} \;p_{\mathrm{thr}}(k;q_1)$, where each $p_{\mathrm{thr}}(k;q_1)$ is a threshold derived to bound Type-I errors at scale $k$. This threshold is estimated non-parametrically through the subsampling procedure detailed in Appendix~\ref{app:psig}.

\item[(ii)] \emph{Global} false--discovery--rate (FDR)
controlled significance levels are obtained by applying the \citet{Benjamini1995}
procedure to the empirical $p$--values $\{p_k\}$.
This controls the expected proportion of spurious discoveries across
the entire range $k = 1,\dots,k_{\max}$ at a prespecified level $q_2$.
\end{itemize}

The two criteria address distinct inferential goals:
the per--scale (per-$k$) criterion provides conservative protection at each
resolution individually, while the FDR criterion accounts for multiple
comparisons across clustering scales.
The per--scale criterion is appropriate when the question is 
\emph{``Is there structure at this particular resolution?''}.
In contrast, global FDR control is relevant when the goal is 
\emph{``Which values of $k$ are significant overall after adjusting for 
the search across many scales?''}.

\section{Applications to synthetic clustered data}
\label{sec:SyntheticClusteredData}

To assess the ability of ElbowSig to detect clustering structure, we generated synthetic datasets in $D$ dimensions from mixtures of $M$ Gaussian components whose centres were placed uniformly at random in the hypercube $[-10,10]^D$. All components share a common standard deviation $\sigma_c$.  

As an illustrative example, Fig.~\ref{fig:2D_Example_Traditional}(a) shows a two-dimensional dataset with $M=3$ centres and $\sigma_c = 1$.
Agglomerative clustering with Ward's minimum variance linkage and Euclidean metric was applied to obtain candidate partitions. Standard cluster-number estimators give widely conflicting results in this setting: the maxima of the Calinski--Harabasz (CH), Davies--Bouldin (DB), and silhouette indices suggest $\hat{k}=9$, $\hat{k}=6$, and $\hat{k}=2$ clusters, respectively (Fig.~\ref{fig:2D_Example_Traditional}(b)-(c)).

\begin{figure}[!t]%
\centering
\includegraphics[scale=0.7]{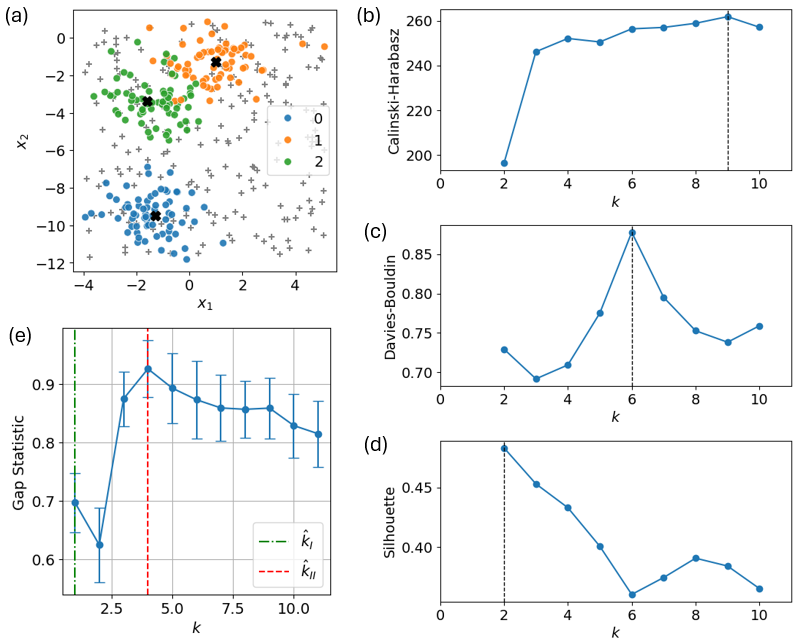}
\caption{
Clustering of two-dimensional data.
(a) Dataset with $N=200$ observations consisting of $M=3$ Gaussian components centred at the locations indicated by crosses, each with spread $\sigma_c = 1$. Solid circles show data points coloured according to their generating component.  Grey crosses show one realization of a bounding-box uniform reference dataset. Partitions with $k \in \{1,2,\dots,10\}$ were evaluated using five standard methods for estimating the number of clusters: (b) Calinski--Harabasz, (c) Davies--Bouldin, and (d) silhouette scores as functions of $k$. Vertical dashed lines indicate the value $\hat{k}$ at which each index attains its maximum. (e) Gap statistic as a function of $k$, with vertical lines marking the optimal number of clusters according to the Gap (I) and Gap (II) criteria described in the text. Error bars indicate the standard error estimate $s_k$  \citep{tibshirani_estimating_2001}.
\label{fig:2D_Example_Traditional}
}
\end{figure}

We also considered two commonly used criteria for selecting the number of clusters based on the gap statistic \citep{tibshirani_estimating_2001} which compares the sum of within-cluster quadratic distances for the observed data with that obtained from $N_R$ unstructured reference datasets. Gap (I) refers to the original selection rule proposed by \citet{tibshirani_estimating_2001}, in which $\hat{k}_I$ is the smallest value of $k$ satisfying 
$
\mathrm{Gap}(k) \;\ge\; \mathrm{Gap}(k+1) - s_{k+1},
$
where $s_{k}$ is the standard error estimate defined in the original work.  Gap (II) denotes an alternative heuristic criterion in which $\hat{k}_{II}$ is chosen as the maximizer of $\mathrm{Gap}(k)$ over the considered range of cluster numbers \citep{Sugar2003, Yan2007}.  Using $N_R = 200$ bounding-box uniform reference datasets, we obtained $\hat{k}_I=1$ and $\hat{k}_{II}=4$ (Fig.~\ref{fig:2D_Example_Traditional}(e)). Using  PCA-aligned uniform reference datasets yields a qualitatively similar picture with $\hat{k}_I=1$ and $\hat{k}_{II}=3$ (Supplementary Fig.~1). 

Fig.~\ref{fig:2D_Example_ElbowSig}(a) displays the dependence of the within-cluster squared distance heterogeneity (inertia) on $k$ for the dataset shown in Fig.~\ref{fig:2D_Example_Traditional}(a). A visual examination might suggest a mild elbow at $k=2$, but no particularly pronounced elbow is evident for this example. In particular, the curve does not exhibit a visually clear inflection at $k=3$, which corresponds to the true number of generating components.

Overall, among the tested traditional cluster-number estimators, only the Gap (II) method with PCA-aligned uniform reference datasets recover the true number of clusters ($M=3$) in this example. The limited separation between clusters~1 and~2 (see Fig.~\ref{fig:2D_Example_Traditional}(a)) likely contributes to the inability of most of the traditional methods studied here to identify the true number of components. However, the precise reasons why different methods fail in different ways are multifaceted and depend on how each criterion balances within-cluster compactness, between-cluster separation, and sensitivity to noise or scale. A detailed analysis of these failure modes falls outside the scope of the present example, whose primary purpose is to illustrate the challenges encountered by conventional cluster-number selection techniques in multi-scale settings.

\begin{figure}[!t]%
\centering
\includegraphics[scale=0.6]{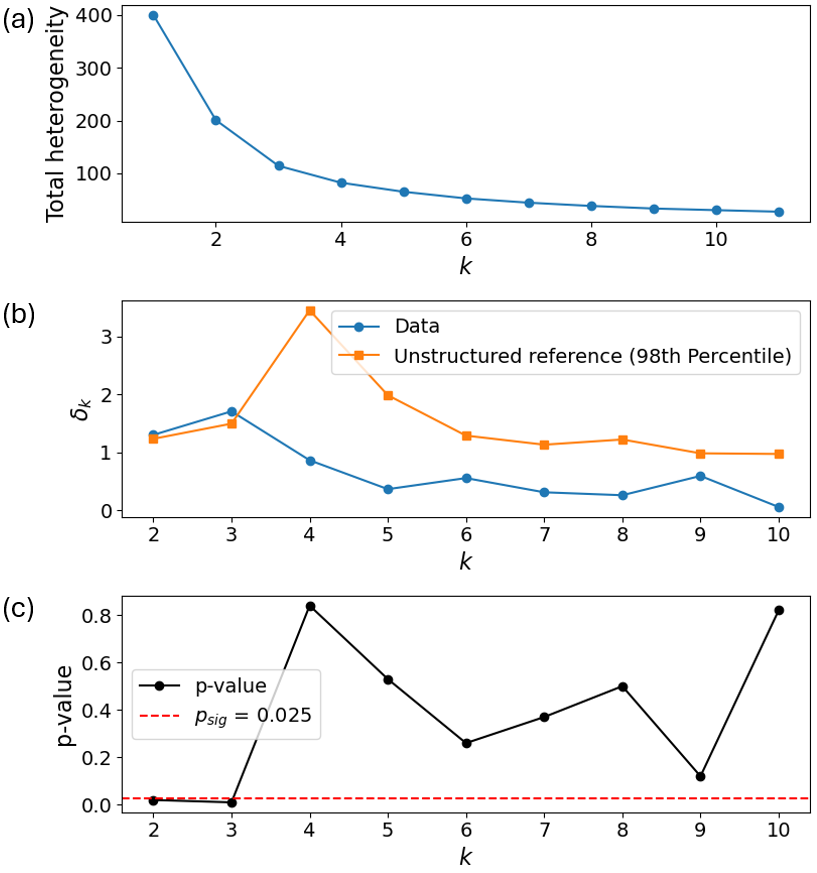}
\caption{Application of ElbowSig to the two-dimensional dataset shown in Fig.~\ref{fig:2D_Example_Traditional}.
(a) Total within-cluster heterogeneity (inertia) as a function of $k$.
(b) Observed elbow statistic (circles) together with the $100 (1 - p_{\text{sig}})$ percentile of the $\delta_k$ distribution obtained from unstructured reference datasets (squares) at level $q_1=0.05$.
(c) Empirical p-values $p_k$ comparing the observed elbow statistic at each $k$ with its reference distribution. The horizontal dashed line indicates the per-scale significance threshold $p_{\text{sig}}$.
\label{fig:2D_Example_ElbowSig}
}
\end{figure}

The elbow statistic $\delta_k^{\text{data}}$ exhibits several local maxima over the explored range of $k$ values (Fig.~\ref{fig:2D_Example_ElbowSig}(b)), which could indicate clustering transitions. However, comparing $\delta_k^{\text{data}}$ with the elbow statistics obtained from $N_R = 200$ uniformly distributed bounding-box reference datasets reveals that only $k=2$ and $k=3$ correspond to partitions that are statistically distinguishable from the reference data when using per-scale control at level $q_1 = 0.05$. Specifically, only for $k=2$ and $k=3$ do the empirical $p$-values satisfy $p_k < p_{\mathrm{sig}}$ (see Fig.~\ref{fig:2D_Example_ElbowSig}(c)). Equivalently, $\delta_k^{\text{data}}$ lies above the $100 (1 - p_{\mathrm{sig}})$-percentile of the reference $\delta_k$ distribution only for these two values of $k$ (see Fig.~\ref{fig:2D_Example_ElbowSig}(b)).

The fact that both $k=2$ and $k=3$ pass the per-scale significance criterion indicates evidence for multi-scale structure in the dataset. In particular, the partition into two clusters captures the dominant large-scale separation, whereas the partition into three clusters reflects an additional, finer-scale subdivision—consistent with the underlying generative model in which one of the large clusters contains two partially overlapping subclusters (clusters~1 and~2 in Fig.~\ref{fig:2D_Example_Traditional}(a)). Rather than forcing a single ``optimal'' value of $k$, ElbowSig identifies resolutions at which statistically meaningful structure is present, thereby providing a more faithful description of the multi-scale organization implicit in the data. 

When ElbowSig is applied using PCA-aligned reference datasets, only the partition with $k = 3$ is found to differ significantly from random data at the per-scale control level $q_1 = 0.05$ (see Supplementary Fig.~2).

When applying global FDR control to the collection of ElbowSig $p$-values $\{p_k\}_{k=1}^{k_{\text{max}}}$, no partition is deemed statistically significant at the level $q_2 = 0.05$ in this example, regardless of whether the reference datasets are drawn from a bounding-box uniform distribution or from a PCA-aligned uniform distribution.


\begin{table}
\centering
\caption{Results for synthetic clustered datasets with $N=200$ observations across dimensions $D \in \{2,5,20\}$, 
with $M$ Gaussian components randomly centred in $[-10,10]^D$ and within-cluster standard deviation $\sigma_c$ as indicated. 
Each entry reports how many of the 100 simulated datasets led a method to select the corresponding number of clusters $k$. Results for the ElbowSig and Gap methods are based on 200 PCA-aligned uniform reference datasets. ElbowSig results are shown for both the individual scale (per-$k$) and global (FDR) significance criteria.  
Gap (I), Gap (II), Calinski--Harabasz (CH), Davies--Bouldin (DB), and the Silhouette index are included for comparison.  
Dashes at $k=1$ for all methods except the gap-based ones indicate that those 
methods cannot return $k=1$ by definition.\label{Table_blobs}}
\small
\begin{tabular}{ccc l *{10}{c}}
\hline
\multirow{2}{*}{$D$} & \multirow{2}{*}{$M$} 
  & \multirow{2}{*}{$\sigma_c$} & \multirow{2}{*}{Method} 
  & \multicolumn{10}{c}{Number of clusters, $k$}\\
\cline{5-14}
 & & & & 1 & 2 & 3 & 4 & 5 & 6 & 7 & 8 & 9 & 10 \\
\hline

\multirow{7}{*}{2} & \multirow{7}{*}{3} & \multirow{7}{*}{1}
& ElbowSig (per-k)  & -- & 28 & 82 & 0 & 3 & 13 & 0 & 1 & 5 & 2  \\
& & & ElbowSig (FDR)  & --   & 25 & 80 & 0 & 1 & 11 & 0 & 1 & 5 & 2 \\
& & & Gap (I)           & 6 & 12 & 81 & 1 & 0 & 0 & 0 & 0 & 0 & 0 \\
& & & Gap (II)          & 0 & 8 & 76 & 1 & 2 & 5 & 0 & 1 & 3 & 4 \\
& & & CH                & -- & 16 & 70 & 0 & 1 & 3 & 0 & 1 & 4 & 5 \\
& & & DB              &  -- & 2 & 1 & 0 & 12 & 30 & 12 & 13 & 14 & 16 \\
& & & Silhouette    &  -- & 35 & 64 & 0 & 0 & 0 & 0 & 0 & 0 & 1 \\

\hline

\multirow{7}{*}{5} & \multirow{7}{*}{3} & \multirow{7}{*}{1}
& ElbowSig (per-k) & -- & 2 & 99 & 0 & 0 & 14 & 2 & 1 & 9 & 4 \\
& & & ElbowSig (FDR)  & --   & 2 & 99 & 0 & 0 & 12 & 1 & 0 & 7 & 3 \\
& & & Gap (I)           & 0 & 0 & 99 & 0 & 0 & 1 & 0 & 0 & 0 & 0 \\
& & & Gap (II)          & 0 & 0 & 66 & 0 & 0 & 15 & 4 & 0 & 6 & 10 \\
& & & CH              &  -- & 1 & 99 & 0 & 0 & 0 & 0 & 0 & 0 & 0 \\
& & & DB              &  -- & 0 & 0 & 0 & 0 & 32 & 7 & 9 & 27 & 25 \\
& & & Silhouette     &  -- & 7 & 93 & 0 & 0 & 0 & 0 & 0 & 0 & 0 \\

\hline

\multirow{7}{*}{5} & \multirow{7}{*}{3} & \multirow{7}{*}{4}
& ElbowSig (per-k) & -- & 13 & 77 & 1 & 0 & 2 & 5 & 4 & 1 & 1 \\
& & & ElbowSig (FDR)  & --   & 7 & 69 & 1 & 0 & 1 & 4 & 2 & 0 & 1 \\
& & & Gap (I)           & 14 & 17 & 69 & 0 & 0 & 0 & 0 & 0 & 0 & 0 \\
& & & Gap (II)          & 2 & 9 & 81 & 6 & 0 & 0 & 2 & 0 & 0 & 0 \\
& & & CH              &  -- & 50 & 50 & 0 & 0 & 0 & 0 & 0 & 0 & 0 \\
& & & DB              &  -- & 4 & 3 & 8 & 12 & 33 & 22 & 5 & 9 & 4 \\
& & & Silhouette     &  -- & 57 & 43 & 0 & 0 & 0 & 0 & 0 & 0 & 0 \\

\hline

\multirow{7}{*}{20} & \multirow{7}{*}{5} & \multirow{7}{*}{1}
& ElbowSig (per-k) & -- & 0 & 0 & 1 & 100 & 0 & 0 & 0 & 2 & 5 \\
& & & ElbowSig (FDR)   & -- & 0 & 0 & 1 & 100 & 0 & 0 & 0 & 1 & 5 \\
& & & Gap (I)           & 0 & 0 & 0 & 0 & 100 & 0 & 0 & 0 & 0 & 0 \\
& & & Gap (II)          & 0 & 0 & 0 & 0 & 97 & 1 & 0 & 0 & 1 & 1 \\
& & & CH               & -- & 0 & 0 & 0 & 100 & 0 & 0 & 0 & 0 & 0 \\
& & & DB               & -- & 0 & 0 & 0 & 0 & 0 & 0 & 1 & 16 & 83 \\
& & & Silhouette     & --  & 0 & 0 & 2 & 98 & 0 & 0 & 0 & 0 & 0 \\
\hline
\end{tabular}
\end{table}

We examined various combinations of the dimension $D$, number of clusters $M$, and cluster spread $\sigma_c$. Table~\ref{Table_blobs} summarizes the results obtained from 100 datasets per setting, using 200 reference datasets for both ElbowSig and the gap-statistic methods.  Agglomerative clustering with Ward's minimum variance linkage and Euclidean metric was used to partition the datasets. ElbowSig consistently identifies the expected number of clusters prescribed by the data-generating model (the value of $M$), while simultaneously revealing additional structure at coarser resolutions: for many datasets, statistically significant clustering is detected for some $k<M$, reflecting cases in which certain Gaussian components overlap and thus form meaningful super-clusters. 

In the results shown in Table \ref{Table_blobs}, cluster overlap is most pronounced for the two-dimensional datasets and for the five-dimensional case with $\sigma_c=4$. In these settings, ElbowSig identifies the true resolution $k=M$ in a smaller fraction of simulations, because the overlap reduces the statistical evidence for a partition that exactly matches the generative components. At the same time, larger overlap increases the number of simulations in which partitions with $k<M$ are declared significant. This reflects the presence of merged ``super-clusters'' in the data rather than a complete absence of structure. As expected, these scenarios also expose the limitations of traditional clustering scores in recovering the true number of components when they overlap.

Statistically significant structure is also sometimes detected by ElbowSig at finer resolutions ($k>M$). In particular, sub-clusters at $k=6$ appear relatively frequently, corresponding to subdivisions of the nominal clusters into two parts; traditional clustering scores likewise often identify $k=6$ as the optimal number clusters. Although such partitions do not represent additional mixture components in the generative model, they indicate detectable within-cluster heterogeneity relative to the null reference and illustrate the multiscale sensitivity of the procedure.

The ability of ElbowSig to recover the true number of mixture components, $M$, is typically superior to that of the CH, DB, silhouette, and gap-statistic methods. Although these traditional criteria may occasionally select a value $\hat{k}$ that coincides with one of the statistically significant values of $k$ identified by ElbowSig, they do not provide any associated measure of statistical confidence. As a result, interpreting the cluster numbers suggested by these scores alone is inherently challenging, particularly in settings where multiple scales of structure are present.

In Table~\ref{Table_methods} we further examine how the ElbowSig results depend on the underlying clustering algorithm, on the construction of the reference datasets, and on whether multiple-testing correction is performed. For agglomerative clustering (Ward’s linkage), $k$-means (initialized with a single random start), FCM (with fuzzy parameter $m=2$ and initialized with a single random start), and GMM, ElbowSig consistently identifies $k=3$ as statistically significant under per-scale testing, in accordance with the true number of mixture components, $M=3$.

Table~\ref{Table_methods} also highlights important differences across clustering methods and reference distributions. When bounding-box uniform reference data are used, agglomerative clustering and $k$-means typically yield strong evidence for the true $k$, whereas FCM tends to declare additional clusters significant. This is likely associated with the fact that the baseline elbow statistic for FCM converges to a positive value, and the fluctuations of the baseline elbow statistic for FCM are smaller than those of the other methods (see Fig.~\ref{fig:Stats_vs_D}(b)). The elbow statistic for the observed data seem to exceed the essentially constant baseline whenever there are local deviations from a truly uniform distribution. In principle, this could reflect the presence of finer-scale structure in the data. However, for relatively small sample sizes, it could also be an artifact of the inability to generate a truly unstructured reference dataset.

The results in Table~\ref{Table_methods} reveal that PCA-aligned reference datasets substantially reduce the number of both fine-scale ($k>M$) and super-cluster ($k<M$) partitions declared significant for all clustering methods except GMM.  
This is a manifestation of the relative nature of cluster structure: The declaration of a pattern as a statistically significant cluster depends on the clustering method and the choice of reference distribution, which defines what is considered ``unstructured'' data.

Overall, the results indicate that ElbowSig reliably recovers the generative number of components across a range of common clustering algorithms, while simultaneously revealing how methodological choices influence the detection of additional multi-scale structure.


\begin{table}[t]
\centering
\small
\caption{Effect of clustering algorithm, reference dataset construction, 
and multiple-testing correction on the ElbowSig outcomes. 
Shown are results from 100 datasets generated from $M=3$ Gaussian components with $\sigma_c=1$ in $D=5$ dimensions. Each realization consists of $N=200$ observations using randomly located centres of the Gaussian components in $[-10,10]^5$. ``BBU'' indicates bounding-box uniform reference datasets, 
and ``PCA'' denotes PCA-aligned uniform reference datasets. The numbers in each 
cell give how many datasets led a method to select the corresponding number 
of clusters $k$.}
\begin{tabular}{l l l *{9}{c}}
\hline
\multirow{2}{*}{Clust. method} & \multirow{2}{*}{Reference} & \multirow{2}{*}{Control} & \multicolumn{9}{c}{Number of clusters, $k$} \\
\cline{4-12}
 & & & 2 & 3 & 4 & 5 & 6 & 7 & 8 & 9 & 10 \\
\hline

\multirow{4}{*}{Agglomerative} 
  & \multirow{2}{*}{BBU} 
    & per-$k$ & 49 & 100 & 2 & 4 & 43 & 9 & 1 & 10 & 6 \\
  &         & FDR     & 47 & 100 & 2 & 3 & 41 & 7 & 1 & 8  & 5 \\
  & \multirow{2}{*}{PCA} 
    & per-$k$ & 2 & 99 & 0 & 0 & 14 & 2 & 1 & 9 & 4 \\
  &         & FDR     & 2 & 99 & 0 & 0 & 12 & 1 & 0 & 7 & 3 \\

\hline

\multirow{4}{*}{K-means} 
  & \multirow{2}{*}{BBU} 
    & per-$k$ & 38 & 99 & 14 & 9 & 25 & 5 & 7 & 1 & 3 \\
  &         & FDR     & 37 & 99 & 11 & 6 & 16 & 1 & 5 & 1 & 3 \\
  & \multirow{2}{*}{PCA} 
    & per-$k$ & 2 & 99 & 6 & 3 & 22 & 2 & 9 & 1 & 3 \\
  &         & FDR     & 2 & 99 & 5 & 2 & 14 & 1 & 7 & 0 & 2 \\

\hline

\multirow{4}{*}{Fuzzy c-means ($m=2$)}
  & \multirow{2}{*}{BBU} 
    & per-$k$ & 31 & 100 & 23 & 41 & 55 & 37 & 55 & 47 & 52 \\
  &         & FDR     & 31 & 100 & 23 & 41 & 55 & 37 & 55 & 47 & 52 \\
  & \multirow{2}{*}{PCA} 
    & per-$k$ & 6 & 100 & 20 & 38 & 52 & 36 & 53 & 43 & 48 \\
  &         & FDR     & 6 & 100 & 20 & 37 & 52 & 36 & 53 & 43 & 48 \\

\hline

\multirow{4}{*}{GMM} 
  & \multirow{2}{*}{BBU} 
    & per-$k$ & 7 & 84 & 1 & 0 & 0 & 4 & 0 & 3 & 0 \\
  &         & FDR     & 1 & 49 & 0 & 0 & 0 & 4 & 0 & 2 & 0 \\
  & \multirow{2}{*}{PCA} 
    & per-$k$ & 20 & 95 & 1 & 0 & 0 & 6 & 1 & 3 & 2 \\
  &         & FDR     & 10 & 70 & 1 & 0 & 0 & 4 & 1 & 2 & 1 \\

\hline
\end{tabular}
\label{Table_methods}
\end{table}

\section{Application to random unstructured synthetic data}
\label{sec:UnstructuredData}

To evaluate the behaviour of ElbowSig in situations where no genuine clustering structure is present, we applied the method to purely unstructured data in dimensions $D=2$ and $D=20$. For each dimension, we generated $100$ independent datasets, consisting either of points drawn uniformly from $[0,1]^D$ or i.i.d.\ Gaussian samples with variance $\sigma_c=1$ in each coordinate. In all cases, ElbowSig was run with agglomerative clustering and $k\leq 10$. For the Gaussian datasets, we present results corresponding to both bounding-box and PCA-aligned reference models. Uniformly distributed datasets are less sensitive to the specific null reference and we only present the results obtained with bounding-box reference datasets.

\begin{table}
\centering
\caption{Results for unstructured data in dimensions $D \in \{2,20\}$ generated from either a Uniform distribution on $[0,1]^D$ or a Gaussian distribution with variance $\sigma_c=1$ in each dimension. For each configuration, 100 independent datasets were simulated. Each entry reports how many of the 100 datasets led a method to select the indicated number of clusters. For ElbowSig, the value displayed in parentheses in the $k=1$ column denotes the number of datasets for which no statistically significant clustering was detected at any tested value $k>1$. Results are shown for both the individual scale (per-$k$) and global (FDR) significance criteria. Gap (I) and Gap (II) correspond to the two gap-statistic selection rules described in the main text. For SigClust, only the number of cases in which the null hypothesis of a single Gaussian component was not rejected is shown. \label{Table_unstructured}}

\small
\begin{tabular}{c l l l *{10}{c}}
\hline
\multirow{2}{*}{D} & \multirow{2}{*}{Data} & \multirow{2}{*}{Method} & \multirow{2}{*}{Reference}
  & \multicolumn{10}{c}{Number of clusters, $k$} \\
\cline{5-14}
 & & &  & 1 & 2 & 3 & 4 & 5 & 6 & 7 & 8 & 9 & 10 \\
\hline

\multirow{5}{*}{2} & \multirow{5}{*}{Uniform}
 & ElbowSig (per-k) & BBU & (78) & 2 & 5 & 4 & 4 & 0 & 1 & 1 & 1 & 4 \\
 & & ElbowSig (FDR) & BBU & (97) & 1 & 0 & 0 & 0 & 0 & 0 & 0 & 0 & 2 \\
 & & Gap (I)          & BBU & 99   & 1 & 0 & 0 & 0 & 0 & 0 & 0 & 0 & 0 \\
 & & Gap (II)         & BBU & 11   & 3 & 14 & 19 & 14 & 6 & 9 & 8 & 3 & 13 \\
 \cline{3-14}
 & & SigClust         & Gaussian & 27   & - & - & - & - & - & - & - & - & - \\ 
\hline

\multirow{9}{*}{2} & \multirow{9}{*}{Gaussian}
 & ElbowSig (per-k) & BBU & (76) & 5 & 9 & 0 & 1 & 3 & 1 & 4 & 0 & 1 \\
& & ElbowSig (FDR)  & BBU & (90) & 4 & 4 & 0 & 0 & 1 & 1 & 0 & 0 & 0 \\
& & Gap (I)            & BBU & 100   & 0 & 0 & 0 & 0 & 0 & 0 & 0 & 0 & 0 \\
& & Gap (II)           & BBU & 100    & 0 & 0 & 0 & 0 & 0 & 0 & 0 & 0 & 0 \\
\cline{3-14}
& & ElbowSig (per-k) & PCA & (82) & 1 & 8 & 0 & 1 & 2 & 2 & 3 & 0 & 1 \\
&& ElbowSig (FDR)  	& PCA	& (94) & 0 & 3 & 0 & 0 & 0 & 2 & 1 & 0 & 0 \\
& & Gap (I)            & PCA &  100  & 0 & 0 & 0 & 0 & 0 & 0 & 0 & 0 & 0 \\
& & Gap (II)           & PCA &  100  & 0 & 0 & 0 & 0 & 0 & 0 & 0 & 0 & 0 \\
\cline{3-14}
 & & SigClust         & Gaussian & 98   & - & - & - & - & - & - & - & - & - \\ 
\hline

\multirow{5}{*}{20} & \multirow{5}{*}{Uniform}
 & ElbowSig (per-k) & BBU & (72) & 4 & 3 & 1 & 5 & 5 & 3 & 3 & 4 & 0 \\
& & ElbowSig (FDR)  & BBU & (91) & 1 & 2 & 0 & 1 & 1 & 1 & 2 & 1 & 0 \\
& & Gap (I)            & BBU & 100  & 0 & 0 & 0 & 0 & 0 & 0 & 0 & 0 & 0 \\
& & Gap (II)           & BBU & 45   & 9 & 9 & 4 & 3 & 3 & 3 & 3 & 3 & 18 \\
\cline{3-14}
 & & SigClust         & Gaussian & 91   & - & - & - & - & - & - & - & - & - \\ 
\hline

\multirow{9}{*}{20} & \multirow{9}{*}{Gaussian}
 & ElbowSig (per-k) & BBU & (75) & 4 & 3 & 2 & 3 & 2 & 4 & 2 & 2 & 3 \\
& & ElbowSig (FDR)   & BBU & (91) & 1 & 0 & 0 & 0 & 1 & 4 & 2 & 0 & 1 \\
& & Gap (I)          & BBU & 100  & 0 & 0 & 0 & 0 & 0 & 0 & 0 & 0 & 0 \\
& & Gap (II)         & BBU & 90   & 4 & 2 & 0 & 1 & 1 & 0 & 0 & 1 & 1 \\
\cline{3-14}
& & ElbowSig (per-k) & PCA & (94) & 0 & 0 & 0 & 2 & 0 & 2 & 0 & 0 & 2 \\
& & ElbowSig (FDR)  & PCA & (100) & 0 & 0 & 0 & 0 & 0 & 0 & 0 & 0 & 0 \\
& & Gap (I)            & PCA & 100 & 0 & 0 & 0 & 0 & 0 & 0 & 0 & 0 & 0 \\
& & Gap (II)           & PCA & 100 & 0 & 0 & 0 & 0 & 0 & 0 & 0 & 0 & 0 \\
\cline{3-14}
 & & SigClust         & Gaussian & 92   & - & - & - & - & - & - & - & - & - \\ 
\hline

\end{tabular}
\end{table}

Table~\ref{Table_unstructured} reports, for each configuration, how often the method identified statistically significant structure under per-scale control at level $q_1=0.05$ or under global FDR control at $q_2=0.05$. The majority of datasets are correctly identified as unstructured, although occasional detections at $k>1$ still arise due to random fluctuations. Indeed, even under the null hypothesis, some slope changes in $H_k$ are statistically expected to exceed the per-scale significance threshold at the nominal rate. Applying global FDR control substantially reduces these detections, again in accordance with the theoretical foundation of the procedure.

The results also reveal systematic differences across reference constructions for Gaussian unstructured data. Bounding-box reference datasets tend to yield more detections of $k>1$ than PCA-aligned references, reflecting again the fact that ``statistical significance'' is always relative to the null model used. In both cases, however, false positives are correctly kept to modest levels.

Gap~(I) performs especially well in this setting, consistently returning $\hat{k}=1$ for almost all unstructured datasets. This behavior is a genuine strength of that particular rule. Nevertheless, Table~\ref{Table_blobs} showed that Gap~(I) performs far less favorably when genuine, possibly multi-scale or overlapping cluster structure is present. Gap~(II), on the other hand, often selects larger $k$ in unstructured data.

For comparison, Table~\ref{Table_unstructured} also reports results from SigClust, which tests the null hypothesis that the data arise from a single Gaussian component \citep{Liu2008,Shen2024}. SigClust is particularly effective at identifying the absence of clustering structure in Gaussian unstructured data. In contrast, it is less effective for uniformly distributed data, which deviate more strongly from the Gaussian null model. The performance of ElbowSig is comparable to that of SigClust for Gaussian data, but outperforms SigClust for uniformly distributed data. Indeed, ElbowSig is expected to be more flexible than SigClust in generic settings where the background structure is not well represented by a single Gaussian model.

\section{Applications to real data}
\label{sec:RealData}

We next evaluated the performance of ElbowSig on five real datasets spanning a diverse range of dimensionalities, sample sizes, and degrees of a priori class separation (Table~\ref{tab:realdata}).

\begin{table}
\centering
\small
\caption{ElbowSig results for real datasets using several clustering algorithms and reference constructions.}
\label{tab:realdata}
\begin{tabular}{lcccc l l l l}
\hline
Data & $N$ & $D$ & Classes & Clustering & Reference & $k$ (per-$k$) & $k$ (FDR)\\
\hline

\multirow{6}{*}{Iris} & \multirow{6}{*}{150} & \multirow{6}{*}{4} & \multirow{6}{*}{3}
  & \multirow{2}{*}{Agglom.} & BBU & 2,3,5 & 2,3,5\\
  & & & & & PCA & 3 & --\\
  & & & & \multirow{2}{*}{k-means} & BBU & 2,3,5 & 2,3,5\\
  & & & & & PCA & 3 & 3\\
  & & & & \multirow{2}{*}{GMM} & BBU & 2,3,5,6,8 & 2,3,5,6,8\\
  & & & & & PCA & 2 & 2\\

\hline

\multirow{6}{*}{Breast cancer} & \multirow{6}{*}{569} & \multirow{6}{*}{30} & \multirow{6}{*}{2}
  & \multirow{2}{*}{Agglom.} & BBU & 2,3,9 & 2,3,9\\
  & & & & & PCA & 2,3 & 3\\
  & & & & \multirow{2}{*}{k-means} & BBU & 2,3 & 2,3\\
  & & & & & PCA & 2,3 & 2,3\\
  & & & & \multirow{2}{*}{GMM} & BBU & 2 & 2\\
  & & & & & PCA & 2 & 2\\

\hline

\multirow{6}{*}{Campylobacter host} & \multirow{6}{*}{673} & \multirow{6}{*}{50} & \multirow{6}{*}{5}
  & \multirow{2}{*}{Agglom.} & BBU & 2--4,6,7 & 2--4,6,7\\
  & & & & & PCA & 6 & --\\
  & & & & \multirow{2}{*}{k-means} & BBU & 2,4,6 & 2,4,6\\
  & & & & & PCA & 6 & 6\\
  & & & & \multirow{2}{*}{GMM} & BBU & 6 & 6\\
  & & & & & PCA & 5,6 & 5,6\\

\hline

\multirow{6}{*}{Human populations} & \multirow{6}{*}{700} & \multirow{6}{*}{258} & \multirow{6}{*}{5}
  & \multirow{2}{*}{Agglom.} & BBU & 2 & 2\\
  & & & & & PCA & -- & --\\
  & & & & \multirow{2}{*}{k-means} & BBU & 4,9 & 4,9\\
  & & & & & PCA & 4,9 & 4,9\\
  & & & & \multirow{2}{*}{GMM} & BBU & 2,4--6,9 & 2,4--6,9\\
  & & & & & PCA & 2,4--6,9 & 2,4--6,9\\

\hline

\multirow{6}{*}{Insulin resistance} & \multirow{6}{*}{302} & \multirow{6}{*}{94} & \multirow{6}{*}{--}
  & \multirow{2}{*}{Agglom.} & BBU & 2,4,6 & 2,4,6\\
  & & & & & PCA & -- & --\\
  & & & & \multirow{2}{*}{k-means} & BBU & 2,4,6 & 2,4,6\\
  & & & & & PCA & 6 & 6\\
  & & & & \multirow{2}{*}{GMM} & BBU & 2 & 2\\
  & & & & & PCA & 2,6 & 2,6\\

\hline
\end{tabular}
\end{table}

The first two datasets are classical benchmarks for clustering, while the remaining three arise from contemporary biological applications:
\begin{itemize}
\item The \emph{Iris} dataset consists of four measurements for \(N = 150\) flowers from three species (\emph{I. setosa}, \emph{I. versicolor} and \emph{I. virginica}) \citep{FISHER1936}.
\item The Wisconsin breast cancer dataset contains \(N = 569\) samples described by diagnostic measurements and is labelled into two classes (benign versus malignant) \citep{Wolberg1990}.
\item The \emph{Campylobacter} dataset comprises genotypes from 673 isolates collected from five host species (chicken, cattle, sheep, pig, and wild bird). Each isolate is characterized using a panel of 50 highly discriminative SNPs selected according to the S1 strategy described in \cite{perez-reche_mining_2020}.
\item The human populations dataset consists of microsatellite measurements from seven geographically separated populations \citep{pemberton_population_2013}: Africa, America, Central/South Asia, East Asia, Europe, the Middle East, and Oceania. We analyzed a reduced subset of \(N = 700\) individuals (100 per population), retaining the 20\% of loci with the highest coefficient of variation, resulting in \(D = 258\) loci.
\item Finally, we considered an effect size profile (ESP) dataset describing associations between insulin resistance levels and 94 serum polar metabolites in non-diabetic individuals from the Danish MetaHIT cohort \citep{Lynch2016, Pedersen2018}. The ESP data were generated using the ESPClust software \citep{Prez-Reche2025} and consist of \(N = 302\) vectors of dimension \(D = 94\), each corresponding to a window defined over BMI, sex, and gut microbiome gene richness and summarizing associations derived from 275 individuals.
\end{itemize}

Across all datasets, ElbowSig was run with agglomerative clustering, $k$-means, and Gaussian mixture models, using both bounding–box uniform and PCA-aligned uniform reference distributions. In line with the tendency observed in Table~\ref{Table_methods} for FCM to declare sub-cluster patterns as significant, we do not report results for this method in the real-data applications, as this behavior may increase the risk of false positives and complicate interpretation.

Under per-scale significance control, ElbowSig frequently identified multiple values of $k$ that were statistically distinguishable from the corresponding reference structure, indicating the presence of clustering structure at more than one resolution.

Several datasets exhibit clear multiscale behavior. For the Iris dataset, significant resolutions consistently include $k=3$, corresponding to the known species structure, as well as coarser ($k=2$) and finer ($k \geq 5$) partitions. The prominence of $k=2$ reflects the well-known partial overlap between \emph{I. versicolor} and \emph{I. virginica}~\citep{Bezdek1974}, while the finer partitions indicate statistical substructure relative to a null reference.
Similarly, the \emph{Campylobacter} host dataset shows significant structure across a range of resolutions, with coarse splits ($k=2$–$4$) reflecting the overlap between some of the host reservoirs (mainly cattle and sheep~\citep{perez-reche_mining_2020}), and finer resolutions ($k \geq 6$) potentially indicating more subtle genotype differentiation.
The human populations dataset likewise displays multiscale structure, with significant $k$ extending beyond the five nominal populations.
Overall, the presence of statistically significant substructure does not necessarily imply biologically meaningful subgroups; assessing their relevance would require further investigation using, e.g., larger datasets.

In contrast, the breast cancer dataset generally favors a smaller and more stable set of significant resolutions, typically centered around $k=2$ or $k=3$. This behavior aligns with the expectation of an approximately binary separation between benign and malignant samples. The insulin resistance dataset exhibits intermediate behavior: while no ground-truth class labels are available, ElbowSig identifies multiple statistically significant resolutions under per-scale control, suggesting the presence of structure at more than one level across different covariate-defined windows. Notably, a coarse partition into $k=2$ clusters used by \citet{Prez-Reche2025} is identified here as statistically significant across several clustering methods.

As expected, the application of global FDR control yields a more parsimonious set of significant resolutions. However, the discrepancies between the per-scale and FDR-adjusted results remain marginal across the studied datasets. Notably, even under the more stringent FDR criterion, the majority of datasets retain at least one statistically significant clustering scale, suggesting that the identified structures are robust to multiple-testing adjustments. 

Table~\ref{tab:realdata} also reveals a high degree of qualitative agreement across clustering algorithms, particularly for dominant resolutions, while differences across baseline reference constructions are more pronounced. PCA-aligned reference distributions tend to yield more conservative decisions than bounding–box uniform references. A similar trend was observed for synthetic data, suggesting that PCA-aligned reference distributions impose a more stringent null.

\section{Conclusion}
\label{sec:Conclusion}

In this work, we introduced ElbowSig, a statistically principled framework that formalizes the widely used elbow heuristic for determining the number of clusters. By interpreting the elbow as a discrete curvature feature of the heterogeneity curve $H_k$, ElbowSig provides a rigorous inferential framework for a traditionally heuristic procedure.

Our results highlight three primary insights. First, ElbowSig reliably detects meaningful clustering structure in both synthetic and real datasets. Second, we demonstrate that many empirical datasets exhibit statistically significant structure at multiple scales, challenging the conventional search for a single ``optimal'' $k$ and highlighting the importance of multiscale data characterization. Third, the framework provides a transparent mechanism for balancing sensitivity to fine-scale structure against robustness to stochastic fluctuations via the choice of reference model and multiplicity control.

Controlled experiments on synthetic data confirm that the elbow statistic adheres to the derived asymptotic null distributions, while gaining power as cluster separation increases. These results further show that ElbowSig maintains reliable error control under both local and global significance criteria. 

A key feature of ElbowSig is its algorithm-agnostic design: it requires only the sequence of heterogeneity values $H_k$, making it compatible with any clustering method for which a suitable heterogeneity function can be defined. This flexibility allows users to tailor the clustering algorithm and reference construction to the specific characteristics of their data and research questions, while still benefiting from a rigorous statistical framework for inference.

\appendix
\section*{Appendix A - Subsampling calibration of $p_{\mathrm{thr}}(k;q_1)$}
\label{app:psig}

For each scale $k=1,\dots,k_{\text{max}}$ and for each repetition $s = 1,\dots,S_{\mathrm{sig}}$:
\begin{enumerate}
\item select at random a fraction $f_{\mathrm{sel}}$ of the $N_R$
reference realizations (yielding $N_{\mathrm{sel}}$ ``test'' realizations);
\item for each selected realization with elbow statistic $\delta_k^{(i)}$, compute the
null leave-one-out p-value as
\[
p_k^{(i)} = \frac{1}{N_R - 1}
\sum_{r \ne i}
\mathbbm{1}\!\left(\delta_k^{(r)} \ge \delta_k^{(i)}\right);
\]
\item define the repetition-specific threshold for scale $k$ as the $q_1$-quantile of the set $\{p_k^{(i)}: i \in \text{ selected test set} \}$
\[
p_k^{(s)}(q_1) = 
\mathrm{Quantile}_{q_1}
\left(
\{p^{(i)}_k : i \in \text{selected test set}\}
\right).
\]
\end{enumerate}

The local threshold at scale $k$ is then
\[
p_{\mathrm{thr}}(k;q_1)
=
\min_{1 \le s \le S_{\mathrm{sig}}}
\;p_k^{(s)}(q_1).
\]

\paragraph{Acknowledgements}
Funding support is acknowledged from the UKRI COVID-19 Longitudinal Health and
Wellbeing National Core Study, a Medical Research Council Fellowship (MR/W021455/1).

\paragraph{Supplementary material}
Supplementary figures are available online.

\paragraph{Data availability}
All the data used in the applications to real experiments are publicly available. The source code ElbowSig is available at https://github.com/fjpreche/ElbowSig.git .

%

\end{document}